\documentclass{article} % For LaTeX2e
\usepackage{nips14submit_e,times}
\usepackage{hyperref}
\usepackage{url}

\usepackage[numbers]{natbib}
\usepackage{enumerate}
\usepackage{enumitem}
\usepackage{bm}
\usepackage{amssymb}
\usepackage{amsthm}
\usepackage{amsmath}
\usepackage{graphicx}
\usepackage{mathtools}
\usepackage[stable]{footmisc}
\usepackage{floatflt}
\usepackage{wrapfig}
\usepackage[boxed, noend, noline]{algorithm2e}
\usepackage{subcaption}
\usepackage{extarrows}
\usepackage{pbox}
\usepackage{framed}
\usepackage{wrapfig}

\allowdisplaybreaks

\newcommand{\z}{{\bf z}}

\newcommand{\x}{{\bf x}}
\newcommand{\X}{{\bf X}}
\newcommand{\w}{{\bf w}}
\renewcommand{\u}{{\bf u}}

\renewcommand{\P}{{\mathcal{P}}}

\newcommand{\E}{\mathbb{E}}
\newcommand{\R}{\mathbb{R}}

\newcommand{\hy}{\hat{y}}

\newcommand{\Lone}{{$L_1$}}
\newcommand{\sign}{{\text{sign}}}

\renewcommand{\P}{{\mathcal{P}}}

\renewcommand{\(}{\left(}
\renewcommand{\)}{\right)}

\newtheorem{theorem}{Theorem}

\newtheorem{lemma}{Lemma}

\title{Accelerated Parallel Optimization Methods for \\ Large Scale Machine Learning}

\author{
Haipeng Luo \\
Princeton University\\
\texttt{haipengl@cs.princeton.edu} \\
\And
Patrick Haffner  \;and\; Jean-Fran\c{c}ois Paiement  \\
AT\&T Labs - Research\\
\texttt{\{haffner,jpaiement\}@research.att.com} \\
}

\nipsfinalcopy % Uncomment for camera-ready version

\begin{document}
\maketitle

\begin{abstract}
The growing amount of high dimensional data in different 
machine learning applications requires more efficient and scalable optimization algorithms.
In this work, we consider combining two techniques,
parallelism and Nesterov's acceleration, 
to design faster algorithms for \Lone-regularized loss.
We first simplify BOOM \cite{MukherjeeCaFrSi13}, a variant of gradient descent,
and study it in a unified framework,
which allows us to not only propose a refined measurement of sparsity to improve BOOM,
but also show that BOOM is provably slower than FISTA \cite{BeckTe09}.
Moving on to parallel coordinate descent methods,
we then propose an efficient accelerated version of Shotgun \cite{BradleyKyBiGu11},
improving the convergence rate from $O(1/t)$ to $O(1/t^2)$.
Our algorithm enjoys a concise form and analysis compared to previous work,
and also allows one to study several connected work in a unified way.
\end{abstract}

\section{Introduction}
Many machine learning problems boil down to optimizing specific objective functions.
In this paper, we consider the following generic optimization problem associated with 
\Lone-regularized loss: % minimization of linear predictors:
%\begin{equation}\label{equ:L1Loss}
$$\min_{\w \in \R^d} F(\w) %= \min_{\w \in \R^d} f(\w) + \lambda\|\w\|_1
= \min_{\w \in \R^d} \sum_{i=1}^n \ell(\x_i^T\w, y_i) + \lambda\|\w\|_1 ,$$
%\end{equation}
where %$(\x_1, y_1), \ldots, (\x_n, y_n)$ 
$(\x_i, y_i)_{i=1,\ldots,n}$ 
represent $n$ training examples of the task, 
each with feature vector $\x_i \in \R^d$ and label/response $y_i$,
and $\ell$ is a smooth convex loss function with respect to its first argument.
This objective function is the heart of several important machine learning problems,
including Lasso \cite{Tibshirani96b} where $\ell(\hy, y) = \frac{1}{2}\|\hy-y\|^2$,
and sparse logistic regression \cite{Ng04} where $\ell(\hy, y) = \ln(1+\exp(-y\hy))$.

Much effort has been put into developing optimization methods for this model,
ranging from coordinate minimization \cite{Fu98}, 
randomized coordinate descent \cite{ShalevshwartzTe11},
stochastic gradient descent \cite{BottouBo08, ShalevshwartzSr08},
%stochastic 
dual coordinate ascent \cite{ShalevshwartzZh14},
to higher order methods such as interior point methods \cite{KimKoLuBoGo07},
L-BFGS \cite{NocedalWr06}, to name a few.
However, the need for faster and more scalable algorithms is still growing 
due to the emergence of applications that make use of massive amount of 
high dimensional data (e.g. \cite{SvoreBu11}).

One direction to design faster algorithms is to utilize parallel computations
on shared memory multi-processors or on clusters.
Some methods parallelize over examples \cite{LangfordSmZi09, MannMcMoSiWa09, ZinkevichWeLiSm10},
while others parallelize over features \cite{BradleyKyBiGu11, RichtarikTa12}. %ScherrerTeHaHa12
As the references of the latter approach argue, it is sometimes more preferable to 
parallelize over features for \Lone-regularized loss,
which will thus be the focus of this work.

Another direction to design more efficient algorithms is to make use of the curvature of the objective function 
to obtain faster theoretical convergence rate.
It is well known that for smooth objective functions, 
vanilla gradient descent only converges at a suboptimal rate $O(1/t)$,
while Nesterov's acceleration technique would allow the optimal rate $O(1/t^2)$ \cite{Nesterov83}.
Recent progress along this line includes generalizing Nesterov's acceleration
to randomized coordinate descent \cite{Nesterov12, LeeSi13}.
Even if the objective is not smooth,
algorithms that still enjoy the same fast convergence rate have been proposed
for functions that are a sum of a smooth part and a simple separable non-smooth part,
such as the \Lone-regularized loss we consider here 
(see for instance the FISTA algorithm \cite{BeckTe09}).

In this work we aim to combine both of the techniques mentioned above,
that is, to design parallelizable accelerated optimization algorithms.
Similar work includes \cite{MukherjeeCaFrSi13, FercoqRi14}.
We start by revisiting and improving BOOM \cite{MukherjeeCaFrSi13}, 
a parallelizable variant of accelerated gradient descent that tries to utilize the 
sparsity and elliptical geometry of the data.
We first give a simplified form of BOOM which allows one to clearly 
see the the connection between BOOM and FISTA, 
and also to study these algorithms in a unified framework.
Surprisingly, we show that BOOM is actually provably slower than FISTA
when data is normalized, which is an equivalent way of utilizing 
the elliptical geometry.
Moreover, we also propose a refined measurement of sparsity that improves
the one used in BOOM.

Moving on to parallel coordinate descent algorithms,
we then propose an accelerated version of the Shotgun algorithm \cite{BradleyKyBiGu11}.
Shotgun converges as fast as vanilla gradient descent 
while only updating a small subset of coordinates per iteration (in parallel).
Our accelerated version even improves the convergence rate 
from $O(1/t)$ to $O(1/t^2)$,
that is, the same convergence rate as FISTA while updating much fewer coordinates per iteration.
Our algorithm is a unified framework of accelerated single coordinate descent \cite{Nesterov12, LeeSi13},
multiple coordinate descent and full gradient descent.
However, instead of directly generalizing \cite{Nesterov12, LeeSi13} or the very recent work \cite{FercoqRi14},
we take a different route to present Nesterov's acceleration technique so that
our method enjoys a simpler form that makes use of only one auxiliary sequence and
our analysis is also much more concise.
We discuss how these algorithms are connected and 
which one is optimal under different circumstances. 
We finally mention several computational tricks to allow highly efficient implementation of our algorithm.

\section{Gradient Descent: Improving BOOM}
\label{sec:BOOM}
In this section, we investigate algorithms that make use of a full gradient at each round.
Specifically, we revisit, simplify and improve the BOOM algorithm \cite{MukherjeeCaFrSi13}.

There are essentially two forms of Nesterov's acceleration technique, 
one which follows the original presentation of Nesterov and 
uses two auxiliary sequences of points, 
and the other which follows the presentation of FISTA and
uses only one auxiliary sequence. 
BOOM falls into the first category. 
To make the algorithm more clear and the connection to other algorithms more explicit,
we will first translate BOOM into the second form.
Before doing so, to make things even more concise we assume that the data is {\it normalized}.
Specifically, let $\X$ be an $n$ by $d$ matrix such that the $i^{\text{th}}$ row is $\x_i^T$.
We assume each column of $\X$ is normalized such that $\sum_{i=1}^n X_{ij}^2 = 1$ for all $j \in \{1, \ldots, d\}$.
It is clear that this is without loss of generality (see further discussion at the end of this section). 

Now we are ready to present our simplified version of BOOM (see Algorithm \ref{alg:AGD}, Option 2).
Here, $\beta$ is the smoothness parameter of the loss function such that 
its second derivative $\ell''(\hy, y)$ (with respect to the first argument) is %always 
upper bounded by $\beta$ for all $\hy$ and $y$. 
For instance, $\beta = 1$ for square loss used in Lasso and $\beta=1/4$ for logistic loss.
$\gamma_t$ is the usual coefficient for Nesterov's technique which satisfies:
%\begin{equation}\label{eq:def_gamma}
$\gamma_t = (1 - \theta_t)/\theta_{t+1}$,  
with $\theta_0 = 0, \theta_{t+1} = \frac{1}{2}(1+\sqrt{1+4\theta_t^2}).$
%\end{equation}
Finally, the shrinkage function $\P$ is defined as $\P_a(\w)_j = \sign(w_j)\max\{|w_j|-a, 0\}$.
%and we slightly abuse the notation so that $\P_a(\w)$ represents coordinate-wise shrinkage operation,
%that is, $\P_a(\w)_j = \P_a(\w_j)$.
One can see that BOOM explicitly makes use of the sparsity of the data, 
and uses $\kappa$, a measurement of sparsity, to scale the gradient.

\begin{figure}[t]
\centering
\SetAlCapSkip{.2em}
%\begin{minipage}{0.49\textwidth}
\IncMargin{.5em}
\begin{algorithm}[H]
\caption{Parallel Accelerated Gradient Descent}
\label{alg:AGD}

\SetKwInOut{Input}{Input}
%\SetKwInOut{Parameters}{Parameters}

\Input{normalized data matrix $\X$, smoothness parameter $\beta$, regularization parameter $\lambda$.}

Set step size $\eta = 1/(c\beta)$, where
\begin{itemize}[leftmargin=20pt]
\item Option 1 (FISTA): $c = \rho$ , where $\rho$ is the spectral radius of $\X^T\X$.
\item Option 2 (BOOM): $c = \kappa$, 
 where $\kappa = \max_i \kappa_i$ and $\kappa_i = |\{j : X_{ij} \neq 0\}|$.
\item Option 3 (Our Method): $c = \bar\kappa$, 
where $\bar\kappa = \max_{j} \sum_{i=1}^n \kappa_i X_{ij}^2 $.
\end{itemize}

Initialize $\w_1 = \u_1.$ \\
\textbf{for} $t = 1, 2, \ldots$ \textbf{do} $\;$(in parallel for each coordinate) \\
$\quad \w_{t+1} = \P_{\lambda\eta} \(\u_{t} - \eta\nabla f(\u_t)\).$ \\
$\quad \u_{t+1} = (1-\gamma_t)\w_{t+1} + \gamma_t\w_t.$ \\
\end{algorithm}
\DecMargin{.5em}
%\end{minipage}
\end{figure}

When written in this form, it is clear that BOOM is very close to FISTA (see Algorithm \ref{alg:AGD}, Option 1).
%Note that $\rho\beta$ is exactly the global smoothness parameter of $f(\w)$, that is, 
%the maximum eigenvalue of $\nabla^2 f(\w)$.
The only difference is that BOOM replaces $\rho$ with the sparsity $\kappa$.
The questions are, whether this slight modification results in a faster algorithm?
Does BOOM converges faster by utilizing the sparsity of the data?
The following lemma and theorem answer these questions {\it in the negative}.

\begin{lemma}\label{lem:kappa}
Let $\rho, \bar\kappa$ and $\kappa$ be defined as in Algorithm \ref{alg:AGD}. Then we have
$  \rho \leq \bar\kappa  \leq \kappa. $
\end{lemma}
\begin{proof}
Let $\z$ be a unit eigenvector of $\X^T\X$ with respect to the eigenvalue $\rho$. One has
$ \rho =  \z^T (\rho\z)  =  \z^T (\X^T\X\z) = \|\X\z\|_2^2 
= \sum_{i=1}^n \(\sum_{j: X_{ij}\neq 0} X_{ij}z_j \)^2. $
By Cauchy-Schwarz inequality, we continue
$$ \rho \leq \sum_{i=1}^n \( \sum_{j:X_{ij}\neq 0} 1^2 \sum_{j:X_{ij}\neq 0} X_{ij}^2z_j^2 \) 
= \sum_{i=1}^n \(\kappa_i \sum_{j=1}^d X_{ij}^2z_j^2  \) 
= \sum_{j=1}^d \(z_j^2 \sum_{i=1}^n \kappa_i X_{ij}^2\)  \leq \bar\kappa \leq \kappa, $$
where the last two equalities make use of the fact that 
$\sum_j z_j^2 = 1$ and $\sum_i X_{ij}^2 = 1$ respectively.
\end{proof}

\begin{theorem}\label{thm:AGD}
Suppose $F(\w)$ admits a minimizer $\w^*$. 
As long as $c \geq \rho$, Algorithm \ref{alg:AGD} insures
$ F(\w_t) - F(\w^*) \leq 2c\beta \|\w_1-\w^*\|_2/t^2 $ after $t$ iterations.
In other words, to reach $\epsilon$ accuracy, Algorithm \ref{alg:AGD} 
needs $T_\epsilon = O(\sqrt{c\beta/\epsilon})$ iterations.
\end{theorem}

We omit the proof of Theorem \ref{thm:AGD} since it is a direct generalization of the original proof for FISTA.
Together with Lemma \ref{lem:kappa}, this theorem suggests that we should always choose FISTA over BOOM
if we know $\rho$, since it uses a larger step size and converges faster.
In practice, computing $\rho$ might be time consuming.
However, a useful side product of Lemma \ref{lem:kappa} is that 
it proposes a refined measurement of sparsity $\bar\kappa$, 
which is clearly also easy to compute at the same time.
We include this improved variant in Option 3 of Algorithm \ref{alg:AGD}.
Note that for any fixed $j$, $X_{1j}^2, \ldots, X_{nj}^2$ forms a distribution,
and thus $\bar\kappa$ should be interpreted as the maximum weighted
average sparsity of the data.

\textbf{Generality of Feature Normalizing.}
One might doubt that our results hold merely because of the normalization assumption of the data.
Indeed, authors of \cite{MukherjeeCaFrSi13} emphasize that one of the advantages of BOOM
is that it utilizes the elliptical geometry of the feature space,  
which does not exist any more if each feature is normalized.
However, one can easily verify that the outputs of the following two methods
are completely identical: 
1) apply BOOM directly on the original data;
2) scale each feature first so that the data is normalized, apply BOOM, and at the end
scale back each coordinate of the output accordingly.
Therefore, feature normalization does not affect the behavior of BOOM at all.
Put it differently, feature normalization is an equivalent way to utilize the elliptical geometry of the data.
Note that the same argument does not hold for FISTA.
Indeed, while our results show that FISTA is provably faster than BOOM,
experiments in \cite{MukherjeeCaFrSi13} show the opposite on unnormalized data.
This suggests that we should always normalize the data before applying FISTA.

\section{Coordinate Descent: Accelerated Shotgun}
Updating all coordinates in parallel is not realistic, 
either because of a limited number of cores in multi-processors or 
communication bottlenecks in clusters. 
Therefore in this section, we shift gears and consider algorithms that 
only update a subset of the coordinates at each iterations.
To simplify presentation, we assume there is no regularization (i.e. $\lambda=0$)
and again data is normalized.
We propose a generalized and accelerated version of 
the Shotgun algorithm\footnote{Note that Shotgun has already been successfully accelerated 
with Coordinate Descent Newton method \cite{BradleyKyBiGu11},  
but no analysis of the convergence rate is provided for this heuristic method.}
\cite{BradleyKyBiGu11} (see Algorithm \ref{alg:AShotgun}). 

\begin{figure}[t]
\centering
\SetAlCapSkip{.2em}
\IncMargin{.5em}
\begin{algorithm}[H]
\caption{Accelerated Shotgun}\label{alg:AShotgun}

\SetKwInOut{Input}{Input}

\Input{number of parallel updates $P$, step size coefficient $\eta$, smoothness parameter $\beta$}

Initialize $\w_1 = \u_1$. \\
\For{ $t = 1,2,\ldots,$} {
pick a random subset $S_t$ of $\{1,\ldots, d\}$ such that $|S_t| = P$.

\For{ $j = 1$ \KwTo $d$} {
$w_{t+1, j} = 
\begin{cases}
u_{t, j} - \frac{\eta}{\beta} \nabla f(\u_t)_j  &\text{if $j \in S_t;$}\\
u_{t, j} &\text{else.}
\end{cases} $
}
$\u_{t+1} = (1 - \gamma_t) \w_{t+1} + \gamma_t \w_{t} + c_t (\u_{t}-\w_{t+1})$.
\label{eq:def_u}
}
\end{algorithm}
\DecMargin{.5em}
\end{figure}

What Shotgun does is to perform coordinate descent on several randomly selected coordinates (in parallel),
with the {\it same step size} as in the usual (single) coordinate descent.
If the number of updates per iteration $P$ is well tuned, 
Shotgun converges as fast as doing a full gradient descent, 
even if it updates much fewer coordinates.
The main difference between shotgun and our accelerated version is that 
as usual accelerated methods, our algorithm maintains an auxiliary sequence $\u_t$
at which we compute gradients.
However, $\u_{t+1}$ is not just $(1-\gamma_t)\w_{t+1} + \gamma_t\w_t$
as in Algorithm \ref{alg:AGD}.
Instead, a small {\it step back} has to be taken, which is reflected in the extra
term $c_t (\u_{t}-\w_{t+1})$ where constant $c_t$ will be specified later in Theorem \ref{thm:AShotgun}.
Intuitively, this step back is to reduce the momentum due to the fact that we are not updating all coordinates.
Another important generalization of Shotgun is that we introduce an extra step size coefficient $\eta$,
which allows us to unify several algorithms (see further discussion below).
Note that $\eta$ is fixed to $1$ in the original Shotgun.
We now state the convergence rate of our Algorithm in the following theorem.

\begin{theorem}\label{thm:AShotgun}
Suppose $F(\w)$ admits a minimizer $\w^*$. 
If $P$ and $\eta > 0$ are such that $\frac{\eta}{2}(1+\sigma) < 1$ where
$\sigma = \frac{(P-1)(\rho-1)}{d-1}$ (recall $\rho$ is the spectral radius of $\X^T\X$), 
constant $\gamma_t$ is defined as in Section \ref{sec:BOOM}, %Eq. \eqref{eq:def_gamma},
and %$c_t$ is defined as follows:
%\begin{equation}\label{eq:def_c}
$c_t = \frac{\theta_t}{\theta_{t+1}} \(1 - \frac{2P}{d}\(1-\frac{\eta}{2}(1+\sigma)\) \),$
%\end{equation}
then the following holds for any $t > 1$,
\begin{equation}\label{eq:rate}
\E_{S_{1:t-1}}[F(\w_t)] - F(\w^*) \leq 
\frac{\beta d^2 \|\w_1-\w^*\|^2}{t^2P^2\eta\(1-\frac{\eta}{2}(1+\sigma)\)}, 
\end{equation}
where the expectation is with respect to the random choices of  subsets $S_1, \ldots, S_{t-1}$. 
\end{theorem}

\begin{proof}
For any iteration $k$,  we first consider the expectation of $F(\w_{k+1})$
conditioning on $S_1, \ldots, S_{k-1}$, which we denote by $\E_k[F(\w_{k+1})]$.
Staring from Lemma 3.3 in \cite{BradleyKyBiGu11}, we have
\begin{align*}
\E_k[F(\w_{k+1}) - F(\u_k)] 
&\leq \frac{P}{d}\sum_{j=1}^d \( -\frac{\eta}{\beta} \nabla F(\u_k)_j^2
+ \frac{\beta}{2}(1+\sigma) (\frac{\eta}{\beta} \nabla F(\u_k)_j )^2 \) \\
&= -\frac{P\eta}{d\beta}\(1 - \frac{\eta}{2}(1+\sigma)\) \|\nabla F(\u_k)\|^2.
\end{align*}
Let $\w \in \R^d$ be an arbitrary point, by the above inequality and the convexity of $F$ we have
\begin{align*}
\E_k[F(\w_{k+1}) - F(\w)] 
&= \E_k[F(\w_{k+1}) - F(\u_k) + F(\u_k) - F(\w)] \\
&\leq -\frac{P\eta}{d\beta}\(1 - \frac{\eta}{2}(1+\sigma)\) \|\nabla F(\u_k)\|^2
+ \nabla F(\u_k)^T(\u_k - \w).
\end{align*}
Also, direct calculations show
$$  \E_k[\w_{k+1}-\u_k] = -\frac{P\eta}{d\beta}\nabla F(\u_k), \quad
  \E_k\left[\|\w_{k+1}-\u_k\|^2\right] = \frac{P\eta^2}{d\beta^2} \|\nabla F(\u_k)\|^2, $$
which with $a = -\frac{\beta}{\eta}\(1 - \frac{\eta}{2}(1+\sigma)\) $ and 
$b = \frac{d}{2P\(1-\frac{\eta}{2}(1+\sigma)\)}$ leads to 
$$
\E_k[F(\w_{k+1}) - F(\w)]  \leq a \E_k \left[ 
\|\w_{k+1}-\u_k\|^2 + 2b(\w_{k+1}-\u_k)^T(\u_k-\w)  \right].
$$

By the law of total expectation, taking the expectation with respect to 
$S_1, \ldots, S_{k-1}$ on both sides and plugging $\w=\w_k$ and $\w=\w^*$ 
respectively gives
\begin{equation}\label{eq:keyeq1} 
\delta_{k+1} - \delta_k \leq a \E_{S_{1:k}} \left[ 
\|\w_{k+1}-\u_k\|^2 + 2b(\w_{k+1}-\u_k)^T(\u_k-\w_k) \right], 
\end{equation}
and
\begin{equation}\label{eq:keyeq2} 
\delta_{k+1}  \leq a \E_{S_{1:k}} \left[ 
\|\w_{k+1}-\u_k\|^2 + 2b(\w_{k+1}-\u_k)^T(\u_k-\w^*) \right],
\end{equation}
where we define $\delta_k  = \E_{S_{1:k-1}}F(\w_k) - F(\w^*)$.
Since $\theta_k \geq 1$ for any $k > 0$, 
we now multiply both sides of Eq. \eqref{eq:keyeq1} by $\theta_k-1$, 
and add the result to Eq. \eqref{eq:keyeq2}, arriving at
$$ \theta_k \delta_{k+1} - (\theta_k-1)\delta_k 
\leq a\E_{S_{1:k}} \left[  \theta_k \|\w_{k+1}-\u_k\|^2  +
2b(\w_{k+1}-\u_k)^T(\theta_k\u_k- (\theta_k-1)\w_k - \w^*) \right]. $$
Multiplying both sides by $\theta_k$ and using the 
fact that $\theta_{k-1}^2 = \theta_k^2 - \theta_k$, we obtain
$$ \theta_k^2 \delta_{k+1} - \theta_{k-1}^2 \delta_k 
\leq a\E_{S_{1:k}} \left[   \|\theta_k(\w_{k+1}-\u_k)\|^2  +
2b\theta_k(\w_{k+1}-\u_k)^T(\theta_k\u_k- (\theta_k-1)\w_k - \w^*) \right]. $$
Let $\w'$ be such that 
\begin{equation}\label{eq:trick}
\w_{k+1} - \u_k = b (\w' - \u_k).
\end{equation}
We continue with
\begin{align*}
 \theta_k^2 \delta_{k+1} - \theta_{k-1}^2 \delta_k 
&\leq  ab^2 \E_{S_{1:k}} \left[   \|\theta_k(\w' -\u_k)\|^2  +
2\theta_k(\w'-\u_k)^T(\theta_k\u_k- (\theta_k-1)\w_k - \w^*) \right] \\
&=  ab^2  \E_{S_{1:k}} \left[  \|\theta_k\w' - (\theta_k-1)\w_k - \w^* \|^2
- \| \theta_k \u_k - (\theta_k-1)\w_k  - \w^* \|^2
\right].
\end{align*}
Note that 
\begin{align*}
&\theta_k \w' - (\theta_k-1)\w_k  \\
=\;& \frac{\theta_k}{b}\w_{k+1} + \theta_k(1-\frac{1}{b})\u_k + (1-\theta_k)\w_k 
\tag{by Eq. \eqref{eq:trick}}\\
=\;& (\theta_{k+1}-1+\theta_k)\w_{k+1} + (1-\theta_k)\w_k +
\theta_k(1-\frac{1}{b})(\u_k-\w_{k+1}) - (\theta_{k+1}-1)\w_{k+1} \\
=\;& \theta_{k+1}\( (1-\gamma_k)\w_{k+1} + \gamma_k\w_k 
+ c_k(\u_k - \w_{k+1}) \) - (\theta_{k+1}-1)\w_{k+1} 
\tag{by the definition of $\gamma_k$ and $c_k$}\\
=\;&  \theta_{k+1} \u_{k+1} - (\theta_{k+1}-1)\w_{k+1},
\tag{by the definition of $\u_{k+1}$}.
\end{align*}
So we finally arrive at 
$$ \theta_k^2 \delta_{k+1} - \theta_{k-1}^2 \delta_k
\leq ab^2 ( v_{k+1} - v_{k}), $$
where $v_k = \E_{S_{1:k-1}} 
\left[\| \theta_k \u_k - (\theta_k-1)\w_k  - \w^* \|^2\right]$. 
Summing these inequalities from $k=1$ to $k= t-1$, and noting that $a < 0$
by the assumption $\frac{\eta}{2}(1+\sigma) < 1$, we have
$$  \theta_{t-1}^2\delta_t \leq -ab^2v_1
=  \frac{\beta d^2\|\w_1-\w^*\|^2}{4P^2\eta\(1-\frac{\eta}{2}(1+\sigma)\)}.   $$
By induction one can verify that $\theta_{t-1} \geq t/2$ for any $t>1$, 
and thus Eq. \eqref{eq:rate} follows.
\end{proof}

We now explain how to interpret this result
and how to choose parameters $P$ and $\eta$.
First, note that to reach $\epsilon$ accuracy (i.e. $\E[F(\w_t)] - F(\w^*) \leq \epsilon$),
Algorithm \ref{alg:AShotgun} requires 
$T_\epsilon = O\(\frac{d}{P}\sqrt{\beta/\(\epsilon \eta\(1 - \frac{\eta}{2}(1+\sigma)\)\)}\)$ iterations.
For any fixed $P$, 
one can verify that $\eta^* = 1/(1+\sigma)$ is the optimal choice for $\eta$ to minimize $T_\epsilon$.
Plugging $\eta^*$ back in $T_\epsilon$ and minimizing over $P$, 
one can check that $P=d$ is the best choice. 
In this case, since $\sigma = \rho-1, \eta^*=1/\rho$ and $c_t = 0$,
Algorithm \ref{alg:AShotgun} actually degenerates to FISTA and $T_\epsilon = O(\sqrt{\rho\beta/\epsilon})$,
recovering the results in Theorem \ref{thm:AGD} exactly.

Of course, at the end it is not the number of iterations, 
but the total computational complexity that we care about most.
Suppose we implement the algorithm without using any parallel computation.
Then as we will discuss later, the time complexity for each iteration would be $O(nP)$ 
(recall $n$ is the number of examples),
and the total complexity $O(T_\epsilon nP)$ is minimized when $\eta = 1$ and $P = 1$, 
leading to $O(nd\sqrt{\beta/\epsilon})$.
In this case, our algorithm degenerates to an accelerated version of randomized (single) coordinate descent.
Note that this essentially recovers the algorithms and results in \cite{Nesterov12, LeeSi13},
but in a much simpler form (and analysis).

Finally, we consider implementing the algorithm using parallel computation.
Note that $\eta^* = 1/(1+\sigma)$ is always at most 1.
So we fix $\eta$ to be $1$ (as in the original Shotgun algorithm),
leading to the largest possible step size $1/\beta$ and a potentially small $P$. 
Ignoring for now hardware or communication limits,
we assume in this case updating $P$ coordinates in parallel costs approximately the same time
no matter what $P$ is. 
In other words, we are again only interested in minimizing $T_\epsilon$.
One can verify that the optimal choice here for $P$ is 
$\frac{2}{3}(\frac{d-1}{\rho-1}+1)$ and $T_\epsilon = O(\rho\sqrt{\beta/\epsilon})$.
This improves the convergence rate from $O(1/\epsilon)$ to $O(\sqrt{1/\epsilon})$ 
compared to Shotgun,
and is almost %as fast as FISTA (with an extra $\sqrt{\rho}$ factor).
as good as FISTA (with much less computation per iteration).

\textbf{Efficient Implementation.}
We discuss how to efficiently implement Algorithm \ref{alg:AShotgun}. %our accelerated version of Shotgun.
1) At first glance it seems that computing vector $\u$ needs to go over all $d$ coordinates.
However, one can easily generalize the trick introduced in \cite{LeeSi13} to update only $P$ coordinates
and compute $\w$ and $\u$ implicitly. 
2) Another widely used trick is to maintain inner products between examples and weight vectors
(i.e. $\X\w$ and $\X\u$), so that computing a single coordinate of the gradient can be done in $O(n)$, 
or even faster in the case of sparse data.
3) Instead of choosing $S_t$ uniformly at random, 
one can also do the following: 
arbitrarily divide the set $\{1,\ldots,d\}$ into $P$ disjoint subsets of equal size in advanced
(assuming $d$ is a multiple of $P$ for simplicity), 
then at each iteration, 
select one and only one element from each of the $P$ subsets uniformly at random to form $S_t$.
This would only lead to a minor change of the convergence results 
(indeed, one only needs to redefine $\sigma$ to be $(\rho-1)P/d$).
The advantage of this approach is that it suggests that we can separate the
data matrix $\X$ by columns and store these subsets on $P$ machines separately
to naturally allow parallel updates at each iteration.

\newpage
\bibliographystyle{plain}
{\bibliography{../../references/ref.bib}}

%\subsubsection*{Acknowledgments}
%Use unnumbered third level headings for the acknowledgments. All
%acknowledgments go at the end of the paper. Do not include 
%acknowledgments in the anonymized submission, only in the 
%final paper. 

\end{document}